\documentclass[nohyperref]{article}

\usepackage{microtype}
\usepackage{graphicx}
\usepackage{comment}
\usepackage{subcaption}
\usepackage{booktabs} 

\usepackage{hyperref}

\usepackage[commentColor=black,beginLComment=/*~, endLComment=~*/]{algpseudocodex}

\usepackage[accepted]{icml2023}

\usepackage{amsmath}
\usepackage{xfrac}
\usepackage{amssymb}
\usepackage{mathtools}
\usepackage{amsthm}
\usepackage{bbding}

\usepackage[capitalize,noabbrev]{cleveref}

\theoremstyle{plain}
\newtheorem{theorem}{Theorem}[section]
\newtheorem{proposition}[theorem]{Proposition}

\newtheorem{corollary}[theorem]{Corollary}
\theoremstyle{definition}

\theoremstyle{remark}

\icmltitlerunning{Neural FIM for learning Fisher Information Metrics from point cloud data}

\usepackage{amsmath,amsfonts,bm}

\def\eqref#1{equation~\ref{#1}}
\def\Eqref#1{Equation~\ref{#1}}

\def\1{\bm{1}}

\def\mA{{\bm{A}}}

\def\mD{{\bm{D}}}

\def\mJ{{\bm{J}}}
\def\mK{{\bm{K}}}

\def\mP{{\bm{P}}}
\def\mQ{{\bm{Q}}}

\def\mU{{\bm{U}}}

\def\mX{{\bm{X}}}

\DeclareMathAlphabet{\mathsfit}{\encodingdefault}{\sfdefault}{m}{sl}
\SetMathAlphabet{\mathsfit}{bold}{\encodingdefault}{\sfdefault}{bx}{n}

\def\gF{{\mathcal{F}}}

\def\gN{{\mathcal{N}}}

\def\gP{{\mathcal{P}}}

\def\gX{{\mathcal{X}}}

\newcommand{\pdata}{p_{\rm{data}}}

\newcommand{\E}{\mathbb{E}}

\newcommand{\R}{\mathbb{R}}

\DeclareMathOperator*{\argmin}{arg\,min}

\begin{document}

\twocolumn[
\icmltitle{Neural FIM for learning Fisher information metrics from point cloud data}




\icmlsetsymbol{equal}{*}
\icmlsetsymbol{eco}{\Cross}

\begin{icmlauthorlist}
\icmlauthor{Oluwadamilola Fasina}{equal,yale_amath}
\icmlauthor{Guillaume Huguet}{equal,mila,mila_math}
\icmlauthor{Alexander Tong}{mila,mila_cs}
\icmlauthor{Yanlei Zhang}{mila,mila_math}
\icmlauthor{Guy Wolf}{mila,mila_math}
\icmlauthor{Maximilian Nickel}{fair}
\icmlauthor{Ian Adelstein}{eco,yale_math}
\icmlauthor{Smita Krishnaswamy}{eco,yale_cs,yale_genetics,yale_amath}
\end{icmlauthorlist}

\icmlaffiliation{yale_amath}{Applied Math Program, Yale University, New Haven, CT, USA.}
\icmlaffiliation{mila}{Mila - Quebec AI Institute, Montreal, QC, Canada.}
\icmlaffiliation{yale_cs}{Department of Computer Science, Yale University, New Haven, CT, USA.}
\icmlaffiliation{yale_genetics}{Department of Genetics, Yale University, New Haven, CT, USA.}
\icmlaffiliation{yale_math}{Department of Math, Yale University, New Haven, CT, USA.}
\icmlaffiliation{mila_math}{Department of Mathematics and Statistics, Universit\'{e} de Montr\'{e}al, Montreal, QC, Canada.}
\icmlaffiliation{mila_cs}{Department of Computer Science and Operations Research, Universit\'{e} de Montr\'{e}al, Montreal, QC, Canada.}
\icmlaffiliation{fair}{FAIR, Meta AI}

\icmlcorrespondingauthor{Smita Krishnaswamy}{smita.krishnaswamy@yale.edu}


\icmlkeywords{Machine Learning, ICML}

\vskip 0.3in
]



\printAffiliationsAndNotice{\icmlEqualContribution, \icmlcoseionr} 

\begin{abstract}

Although data diffusion embeddings are ubiquitous in unsupervised learning and have proven to be a viable technique for uncovering the underlying intrinsic geometry of data, diffusion embeddings are inherently limited due to their discrete nature. To this end, we propose \emph{neural FIM}, a method for computing the Fisher information metric (FIM) from point cloud data - allowing for a continuous manifold model for the data.  Neural FIM creates an extensible metric space from discrete point cloud data such that information from the metric can inform us of manifold characteristics such as volume and geodesics. We demonstrate Neural FIM's utility in selecting parameters for the PHATE visualization method as well as its ability to obtain information pertaining to local volume illuminating branching points and cluster centers embeddings of a toy dataset and two single-cell datasets of IPSC reprogramming and PBMCs (immune cells).

\end{abstract}

\section{Introduction}

An important goal of unsupervised learning is understanding the underlying shape or geometry of data \cite{bronstein2017geometric,cheng2019diffusion,de2017diffusion,tsitsulin2019shape}. A key paradigm here is the manifold assumption which hypothesizes that high dimensional data, particularly from scientific domains, lies on a lower dimensional smoothly varying manifold (see \citet{huguet_manifold_2022,he2014intrinsic,bhaskar2022diffusion,lin2008riemannian}). Prior methods for learning data manifolds use data affinity kernels that  compute a pairwise distance matrix from a data set, then pass the distances through a kernel function (such as a Gaussian kernel function) to convert distances to affinities \cite{belkin2003laplacian,bunte2012stochastic,mika1998kernel}. Eigenvectors of such an affinity matrix give the data a manifold-intrinsic coordinate representation.  This method, while successful in some respects, has a key disadvantage that it is implicitly biased by the particular sampling of the data that is given, together with its irregularities. Further, there is usually no straightforward way of extending such manifold coordinate representations to unseen points. 

Here we propose a neural-network based method of directly learning a Riemannian metric for data called the neural FIM. Loosely speaking, a Riemannian metric is an infinitesimal generator of manifold-intrinsic length and volume, based on an inner product structure on the tangent space of every point. Typically, a Riemannian metric cannot be learned from discrete data as there is no continuous model of the manifold. Although \citet{bengio2003out,schoeneman2017error,law2006incremental,dadkhahi2017out} developed methods for extending coordinate manifold representations to unseen points, none of these methods admit continuous manifold models and involve expensive computations. Neural FIM is able to learn a continuous manifold model by using a neural network to embed data points into a latent space and creates a continuous implicit model of the data from which we can compute the metric.

The specific Riemannian metric we aim to learn the data manifold is the Fisher Information Metric (FIM). This type of metric is defined on {\em statistical manifolds}, manifolds where each datapoint is a probability distribution \cite{lauritzen1987statistical,lafferty2005diffusion,noguchi1992geometry}. We obtain such a pointwise probability distribution on point cloud data by way of a data diffusion operator, as first defined in the seminal work on diffusion maps~\cite{coifman_diffusion_2006}. After an affinity kernel is computed, it is row normalized to a stochastic matrix. This normalized matrix is treated as a Markovian operator which defines a random walk or a diffusion on the data. We associate each data point to the transition probability distribution given by its row of the stochastic matrix, and thus realize the point cloud data as a statistical manifold. 

By utilizing a mathematical connection between the differential form of Jensen-Shannon Distance (JSD) (the square-root of the standard Jenson-Shannon (JS) divergence) and neural FIM, we derive a method of training the neural FIM on the basis of distances between PHATE embeddings that use JSD.  An advantage of this approach is that---similar to PHATE---the embedding is globally contextualized due to the information-theoretic distances that are computed. We can then use the FIM to compute geometric quantities on the statistical manifold such as length and volume.  We show how the geodesic or Fisher-Rao distance between pairs of points can be computed using an auxiliary neural Ordinary Differential Equation (ODE) network~\cite{chen2018neural}. This distance can be used for novel embeddings and downstream tasks. The magnitude of the volume element captures local distinguishability and can be used to reveal branching points in hierarchical data or decision boundaries in classification problems. 

We showcase our results on three types of tasks. First, we show how to use the FIM to explore the space of parameters for the PHATE embedding method. Here, the statistical manifold is created from the diffusion operator resulting from various embedding parameters (on the same dataset). In specific, we explore selection of the time-of-diffusion and bandwidth variables. The second task involves computing the FIM of 3 different datasets: a toy tree dataset, an IPSC reprogramming mass cytometry dataset \cite{zunder2015continuous} and a pbmc single cell RNA-sequencing dataset~\cite{PBMCs}. These statistical manifolds correspond to transition probability distributions of each datapoint within the dataset. Both the neural FIM embeddings and information from the FIM including volume and trace are shown for each of the three datasets. We see that the volume highlights freedoms of movements with branchpoints having higher volume. Finally, we utilize the neural ODE network to compute geodesic paths within the embedding between points. First we show this on data sampled from a sphere, and then on the IPSC dataset. Remarkably, in the IPSC dataset the geodesic follows the path of reprogramming of a cell from it starting state.\footnote{Code is available at: \url{https://github.com/guillaumehu/phate_fim}}

The key contributions of this work include: 
\begin{itemize}
    \item Conceptually connecting data diffusion embeddings with statistical manifolds in order to derive a manifold model of the data, complete with a continuously-defined Fisher Information Metric tensor.  
    \item Proposing the neural FIM method for extensible FIM computations (i.e., extensible to unseen data) trained by using Jensen-Shannon divergence between data diffusion probabilities. 
    \item Proposing a neural-ODE based method for computing geodesic paths and distances based on the FIM. 
    \item Showcasing the use of neural FIM in selecting parameters, visualizing single cell data, and locally extracting information about the volume, trace, and eigenspectrum of the metric to understand the underlying manifold geometry of the data.
\end{itemize}

\section{Background}

A useful assumption in manifold learning is that data measured in a high-dimensional ambient space originates from an intrinsic low-dimensional manifold. The manifold assumption asserts that if $\mathcal{M}^d$ is a hidden $d$ dimensional manifold, it is observable by a collection of $n \gg d$ nonlinear functions $f_1,\ldots,f_n : \mathcal{M}^d \to \R$ which enable its immersion in a high dimensional ambient space as $F(\mathcal{M}^d) = \{\mathbf{f}(z) = (f_1(z),\ldots,f_n(z))^T : z \in \mathcal{M}^d \} \subseteq \R^n$ from which data is collected. Conversely, given  data $X = \{x_1, \ldots, x_N\} \subset \R^n$ of high dimensional observations, manifold learning methods assume the data originates from a sampling $Z = \{z_i\}_{i=1}^N \subset \mathcal{M}^d$ of the underlying manifold via $x_i = \mathbf{f}(z_i)$,  and aim to learn a low dimensional intrinsic representation that approximates the manifold geometry of~$\mathcal{M}^d$.

\subsection{Data Diffusion}\label{sec: data_diffusion}

A popular class of methods for manifold learning uses a data diffusion operator, which models data based on transition or random walk probabilities through the data. Methods that use a data diffusion operator include diffusion maps \cite{coifman_diffusion_2006}, PHATE \cite{moon_visualizing_2019}, tSNE \cite{van_der_maaten_visualizing_2008}, and diffusion pseudotime \cite{haghverdi2016diffusion}.  One can learn the manifold geometry with data diffusion by first computing local similarities defined via a kernel $\mathcal{K}(x,y)$, $x,y \in F (\mathcal{M}^d)$. 
 We note that a popular choice for a kernel is the Gaussian kernel $\mathcal{G}(x,y)=\exp (-\|x-y\|^2 / \sigma)$, where $\sigma > 0$ is interpreted as a user-configurable scale parameter. However, this choice encodes sampling density information together with local geometric information.

 To construct a diffusion geometry that is robust to sampling density variations, we use an anisotropic kernel $\mathcal{K}(x,y) = \frac{\mathcal{G}(x,y)}{\|\mathcal{G}(x,\cdot)\|_1^{\alpha} \|\mathcal{G}(y,\cdot)\|_1^{\alpha}},$ where $0 \leq \alpha \leq 1$ controls the separation of geometry from density, with $\alpha = 0$ yielding the classic Gaussian kernel, and $\alpha = 1$ completely removing density and 
 essentially providing uniform sampling of the manifold. Finally, we row-normalize $\mathcal{K}$  to define transition probabilities $p(x,y) = \mathcal{K}(x,y)/\|\mathcal{K}(x,\cdot)\|_1$ and define an $N \times N$ diffusion matrix $\mathbf{P}_{ij} = p(x_i, x_j)$ that describes a Markovian diffusion  over the the data.

\subsection{PHATE}
There are several dimensionality reduction methods that render data into 2-D visuals, such as PCA, tSNE~\cite{van_der_maaten_visualizing_2008}, and UMAP~\cite{mcinnes_umap_2018}. However, these methods fail to preserve the global manifold structure of the data and are not robust to noise. PCA cannot denoise in non-linear dimensions, and tSNE/UMAP effectively only constrain for near neighbor preservation---losing global structure. This motivated the development of a method of dimensionality reduction that retains manifold structure and denoises data \cite{moon_visualizing_2019}. 
 
PHATE also builds upon the diffusion-based manifold learning framework from \citet{coifman_diffusion_2006}, and involves the creation of a diffused Markov transition matrix from  data,  $\mathbf{P}$. PHATE collects all of the information in the diffusion operator into two dimensions such that global and local distances are retained. To achieve this, PHATE considers the $i$th row of $\mathbf{P}^t$ as the representation of  the $i$th datapoint in terms of its $t$-step diffusion probabilities to {\em all} other datapoints. PHATE then preserves a novel distance between two datapoints,  based on this representation called {\em potential distance (pdist)}. Potential distance is an $M$-divergence between the distribution in row $i$, $\mathbf{P}^t_{i,.}$ and the distribution in row $j$, $\mathbf{P}^t_{j,.}$. These are indeed distributions as $\mathbf{P}^t$ is Markovian: 

\vspace{-3mm}

\begin{equation}
pdist(i,j)=\sqrt{\sum_k (\log(P^t(i,k))-\log(P^t(j,k))^2}
\end{equation}

The $\log$ scaling inherent in potential distance effectively acts as a damping factor which makes faraway points similarly equal to nearby points in terms of diffusion probability. This gives PHATE the ability to maintain global context. The paper also allows for other types of symmetric divergences such as the JS divergence, which we use in our work to train neural networks.

These potential distances are embedded with metric MDS as a final step to derive a data visualization. \citet{moon_visualizing_2019} have shown that PHATE outperforms tSNE \cite{van_der_maaten_visualizing_2008}, UMAP \cite{mcinnes_umap_2018}, force-directed layout, and 12 other methods on the preservation of manifold affinity, and adjusted rand index on clustered datasets, in a total  of 1200 comparisons on synthetic and real datasets.

\subsection{Information Geometry and Fisher Information}

Information geometry \cite{amari2016information,nielsen2020elementary,arwini2008information,li2022prior,lin2021wasserstein} combines statistics and differential geometry to study the geometric structure of statistical manifolds. A statistical manifold is a Riemannian manifold $(M,g)$ where every point in the space $p \in M$ is a probability distribution. 
The Fisher Information Metric (FIM) is the standard Riemannian metric on statistical manifolds and measures the distinguishability between points on the manifold (probability distributions).

In the Riemannian setting one endows a smooth manifold $M^{n}$ with geometry by defining at each point $p \in M^n$ an inner product $g_p(\cdot, \cdot): T_{p}M \times T_{p}M \to \mathbb{R}$, where $T_{p}M$ represents the tangent space of $M$ at $p$. The collection of inner products $g$  defines a Riemannian metric on the manifold $M$.

On a statistical manifold $M^n$ parameterized by the coordinates $\theta=(\theta_1, \ldots, \theta_n)$ we have that points are distributions $p(x,\theta) \in M$ over some common probability space (or data set) X. 
A Riemannian metric one can compute on a statistical manifold is the Fisher Information Metric (FIM): 
\begin{equation}
\label{eqn1}
I_{ij}(\theta) = \int_{X}\frac{\partial \log p(x,\theta)}{\partial\theta_{i}}\frac{\partial \log p(x,\theta)}{\partial\theta_{j}}p(x,\theta)dx
\end{equation}
One can use this metric to compute geometric quantities on the statistical manifold such as length or volume. Generically, for a parameterized curve on a Riemannian manifold $c \colon [a,b] \to (M,g)$ its length is given by $$L(c)= \int_a^b |\dot{c}(t)| ~dt = \int_a^b \sqrt{ g_{c(t)} (\dot{c}(t), \dot{c}(t) ) } ~dt.$$ Volume of the manifold is given by $$V(M) = \int_M \sqrt{|det(g)|} ~d\theta.$$ These geometric quantities can be used to provide insight into the original data space $X$. We note that access to a Riemannian metric theoretically provides access to its associated Riemann curvature tensor. However the Riemann curvature tensor is defined in via the metric's unique Levi-Civita connection, an object that we do not explore here.

{\bf Example}: We consider a family of distributions $\gF_\Theta$ parameterized by a parameter space $\Theta$. For the Gaussian family the parameter space is $\Theta = \{(\mu,\sigma): \mu\in\mathbb{R}, \, \sigma \in \mathbb{R}^+\}$, and the family is defined by $\gF_\Theta = \{\gN(\mu,\sigma): (\mu,\sigma)\in\Theta\}$. This parameterization turns the Gaussian family into a 2-dimensional statistical manifold, with FIM $$I(\mu,\sigma)=\begin{bmatrix}
1/\sigma^2 & 0\\
0 & 2/\sigma^2
\end{bmatrix}$$

We see that the Gaussian family admits a hyperbolic geometry, where the distance between distributions with fixed differences in means increases as their variance decreases. This example illuminates the main interpretation of the FIM and its associated (Fisher-Rao) distance:~the more distinguishable are two distributions (say in terms of inference) the larger their Fisher-Rao distance. The FIM locates Gaussians with small variance at greater distances in the statistical manifold.

The FIM has a connection to other information theoretic quantities, in particular to Kullback-Liebler (KL) and Jensen-Shannon (JS) divergences, as shown in \citet{crooks2007measuring}. The FIM is an infinitesimal version of the KL-divergence.  Key results are included below for reference. 

\begin{theorem}{\citep[from][]{crooks2007measuring}} \label{thm:JS} The infinitesimal Jensen-Shannon divergence, $dJS=JS(p, p+dp)=\frac{1}{8} \sum_i \frac{(dp_i)^2}{p_i}$ is equal to the FIM, $ \frac{d c^i}{dt} I_{ij}(c) \frac{d c^j}{dt} = \sum_x \frac{1}{p(x)} [\frac{dp(x)}{dt}]^2$. 
\end{theorem}

One extends this infinitesimal result to more global objects by integrating over parameterized paths. 

\begin{corollary}{\citep[from][]{crooks2007measuring}}
The length (with respect to the FIM) of a parameterized path $c(t)$ equals the total Jensen-Shannon divergence over the curve: 
\begin{equation}\label{eq:JS}
\int_a^b \sqrt{ \frac{\partial c^i}{\partial t} I_{i,j} \frac{\partial c^j}{\partial t}}~dt = \sqrt{8} \int_a^b d\sqrt{JS} 
\end{equation}
where $d\sqrt{JS}$ is the infinitesimal change in the Jensen–Shannon divergence along $c(t)$. 
\end{corollary}

\begin{figure*}[!ht]
    \centering
    \includegraphics[width=14cm]{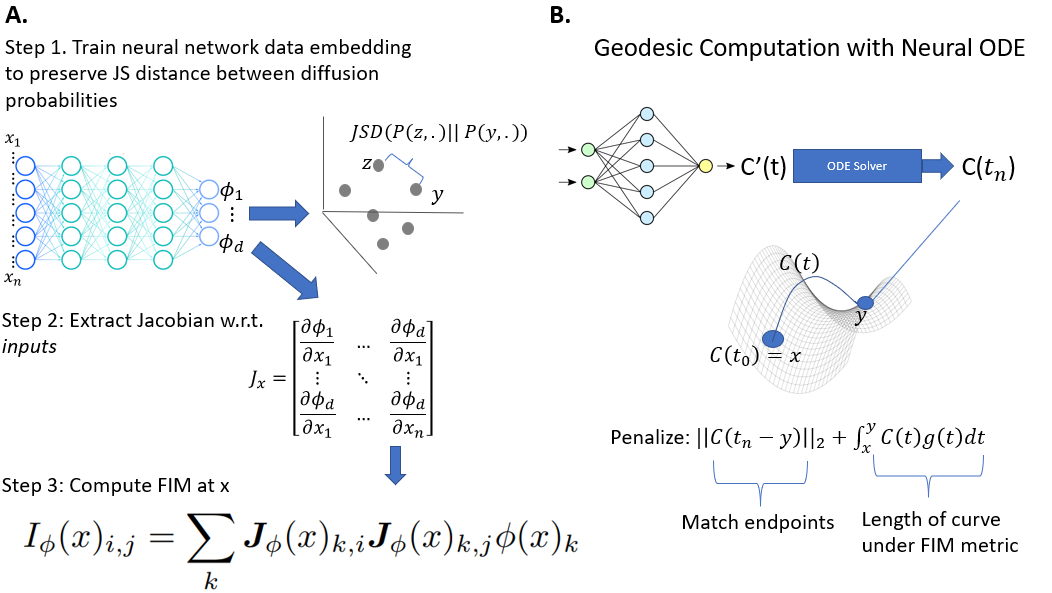}
    \caption{Schematic of neural-FIM which is used to generate a continuous FIM embedding (A) and schematic of neural-ODE used to find geodesics with FIM (B).}
    \label{fig:FIM_schematic}
\end{figure*}

\section{Methods}

\subsection{Neural FIM}

To approximate a continuous FIM from a finite set of distributions, we approximate the family of distributions via a neural network, whose Jacobian we use to evaluate the FIM in a continuous manner. The basic framework, shown in Figure \ref{fig:FIM_schematic}, consists of two parts.
The first part (Figure \ref{fig:FIM_schematic}A) is a neural network trained to match Jensen-Shannon Distances (explained below) from which a Jacobian is extracted for FIM computation. The second part (Figure \ref{fig:FIM_schematic}B) is a neural ODE network that computes geodesic paths, i.e., shortest length paths on data manifolds. 

We consider the dataset to be a point cloud $X\in\gX$ of size $n$ from a sigma finite distribution $q$. The first step is to translate such a point cloud into a family of distributions. To do so, we construct an affinity graph from the point cloud, and its diffusion operator $\mP_n^t$. Each row of the diffusion operator is a distribution (probability mass function) that describes the transition probabilities of a random walk on the data; it thus defines a map $x \mapsto \mP_n(x,\cdot)$, where $\mP_n(x,\cdot)$ is the row corresponding to the observation $x$. The construction of this map is summarized in Algorithm~\ref{alg:phate_fim} (see Appendix). Further, we assume that $x\mapsto \mP_n(x,\cdot)$ is differentiable.

We first consider training a neural network $\phi:\mathbb{R}^d\to \gP(Z)$, where $\gP(Z)$ is the space of probability mass functions on a latent space $Z\subset\mathbb{R}^n$, minimizing the loss function 
\begin{align}
\label{eq:vanillaloss}
L(\phi) &= \E_{x \sim \pdata} \| \phi(x) - \mP_n(x,\cdot)  \|_2\\
&+ \| \nabla_x\phi(x) - \nabla_x \mP_n(x,\cdot)\|_2\notag
\end{align}

This renders the data set into a statistical manifold consisting of $n$ points with each point defining a probability distribution 
in $\mathbb{R}^n$ where $n$ is the dimensionality of the last layer. We can then utilize the Jacobian of this embedding with respect to the inputs to obtain the partial derivatives required for the computation of an FIM at any $x\in\mathbb{R}^d$ via 
\begin{equation}\label{eq:FIM}
    I_\phi(x)_{i,j} = \sum_k  \mJ_{\phi}(x)_{k,i} \mJ_{\phi}(x)_{k,j} \phi(x)_k
\end{equation}
 where $\mJ_\phi$ is the Jacobian matrix of $\phi$ with respect to the input variables. Notably, $\mJ_\phi$ is \emph{not} the Jacobian used for the training of the neural network, i.e., the Jacobian with respect to parameters such as weights and biases of the neural network. However, this training method requires that the dimensionality of the last layer would be very high, and that we train to match data derivatives which hard to obtain. Thus we do not use this loss in  practice. 

We instead offer a much more efficient alternative: we reduce the data to an arbitrarily low $m$ dimensional latent space $Z$ by training the neural network to match the Jensen-Shannon divergence between rows of the distribution, which would be similar to PHATE \cite{moon_visualizing_2019} distance using this alternative divergence:
\begin{align*}
\scriptsize
&JS(\mP_n^t(i,\cdot),\mP_n^t(j,\cdot)):= \\ & \tfrac{1}{2} KL((\mP_n^t(i,\cdot)||M)+KL((\mP_n^t(j,\cdot)||M),
\end{align*}
where $M := (1/2) (\mP_n^t(i,\cdot)+\mP_n^t(j,\cdot))$.
 
Since we match only distances here, there is no restriction on the dimensionality of the output space $Z$. We achieve this by using this alternative loss function: 
\begin{align}
\label{eq:JSLoss}
L_{JS}(\phi) &:=  \E_{x \sim \pdata} \Big\|  \sqrt{JS(\phi(x), \phi(y))} \nonumber \\ 
&- \sqrt{JS(p(x), p(y))} \Big\|_2^2 
\end{align}

Below, we show that in either case our neural network Jacobians can be used to compute the FIM of datapoints within the manifold of the data. 

\begin{proposition}
Assume for any $x\in X$, we have uniform convergence of $\lim_{n\to\infty}\mP_n(x,\cdot) = P(x,\cdot)$ where $\mP_n$ and $P$ are Markov operators with compact support. If $L(\phi_n) \to 0$ uniformly, then we have $\lim_{n\to\infty}\nabla_x \phi_n(x) = \nabla_x  P(x,\cdot)$. 
\end{proposition}
\begin{proof}

Since $\mP_n$ is continuous and has compact support, by the universal approximation theorem \citep{Cybenko1989}, $\mP_n$ can be approximated by a feed forward neural network with a finite number of neurons. By the definition of loss function in Equation~\ref{eq:vanillaloss}, for a fixed $n$, converging uniformly to 0 implies $\phi_n(x)$ converging uniformly to $\mP_n(x, \cdot)$.
Because the convergence of derivatives is also uniform, we can interchange the limit and the derivative, obtaining
\begin{align*}
    \lim_{n\to\infty}\nabla_x \phi_n(x)_j &=  \nabla_x \lim_{n\to\infty} \phi_n(x)_j \\
    &= \nabla_x \lim_{n\to\infty}  \mP_n(x,x_j)\\
    &= \nabla_x  P(x,x_j). 
\end{align*}%
\vspace{-1mm}
 \end{proof}
 
The above proposition shows that a neural network trained to match diffusion probabilities as in Equation \ref{eq:vanillaloss} will have its derivatives as described in in Equation \ref{eq:FIM}. Thus, the neural network can be used to compute the partial derivatives needed to compute the FIM. 
Directly enforcing the loss is simple, however it requires the network output size to scale linearly with the dataset. To allow for FIM to be continuous, we require that our neural-FIM embeddings to be continuously differentiable.

 If instead of using the loss function from Equation \ref{eq:vanillaloss} we use the alternative loss function from Equation \ref{eq:JSLoss}, we still converge to the FIM as below. 

 \begin{proposition} 
Assume that $p$ has compact support. As $|X| \to \infty$, if $L_{JS}(\phi)$ converges to 0, then $I_\phi(x) = I_p(x)$ for all $x \in X$.
\end{proposition}
\begin{proof}
Since $L_{JS}$ converges to 0, for an infinitesimal $\Delta x$,
\begin{align*}
JS(\phi(x), \phi(x + \Delta x)) = JS(p(x), p(x + \Delta x)).
\end{align*}
For any $C^1$ path $c$ between $x$ and $x + \Delta x$, we can apply Theorem~\ref{thm:JS} twice, yielding
\begin{align*}
8  \frac{d c}{dt} I_\phi(c) \frac{d c}{dt} &=JS(\phi(x), \phi(x + \Delta x))
\\
&= JS(p(x), p(x + \Delta x)) \\
&= 8  \frac{d c}{dt} I_p(c) \frac{d c}{dt},
\end{align*}
which implies $I_\phi(x)= I_p(x)$.
\end{proof}

Fixing the aforementioned embeddings to embeddings generated using PHATE, we exploit the fact that
\begin{align}\|\mathrm{PHATE}_{JSD}(x) - \mathrm{PHATE}_{JSD}(y)\|_2 \\
= JSD(P(x, \cdot),P(y, \cdot)),\notag
\end{align}
meaning we can use the loss
\begin{align}
L(\phi) = \E_{x \sim \pdata} \Big\|  \sqrt{JS(\phi(x), \phi(y))}\\ -\|\text{PHATE}_{JSD}(x) - \text{PHATE}_{JSD}(y)\|_2 \Big\|_2^2\notag,
\end{align}
which motivates the use of PHATE with a Jensen-Shannon divergence MDS step. Overall, this analysis also connects the dimensionality reduction method PHATE with neural FIM in that the former is essentially a discrete version of the latter.

\subsection{Geodesic optimization with Neural ODEs}

Using the Neural FIM we can compute various Riemannian quantities to describe and understand our dataset. The most important quantity is the geodesic (manifold-intrinsic) distance between datapoints using the FIM. For the FIM the length of the geodesic is also known as the Fisher-Rao distance. In order to compute this we use a neural ODE that optimizes over all paths between datapoints in order to minimize path length.

Given a Riemannian manifold $(\mathcal{M},g)$ the length of a $C^1$-curve $\gamma \colon [a,b]\to \mathcal{M}$ is 
$ L(\gamma)=\int_{a}^{b}\sqrt{ (\frac{d \gamma}{dt})^T\cdot g_{\gamma(t)}\cdot\frac{d \gamma}{dt} } ~ dt.$
For two distributions $p_{\theta_1}$ and $p_{\theta_2}$, a path from $p_{\theta_1}$ to $p_{\theta_2}$ can be obtained by parameterizing a function $\frac{d\gamma}{dt}=f_\theta(t,\gamma)$  in parameter space so that 
$$\hat{\theta}_2=\gamma(b)=\theta_1+\int_{a}^{b}f_\theta(t,\gamma(t))~dt$$
Among all the paths parameterized by $f_\theta$, penalizing the length of the curve and the prediction loss $||\hat{\theta}_2-\theta_2||_2$ gives the geodesic path, i.e., 
\begin{equation}
    \label{eq:ode_loss}
  \argmin_\theta ~ \lambda \|\hat{\theta}_2 - \theta_2\|_2^2 
    + \int_a^b \sqrt{ f_\theta ^T \cdot g_{\gamma(t)} \cdot f_\theta } ~ dt.
\end{equation}

Thus this network queries the neural FIM network in order to optimize path length based on the FIM. 

\section{Empirical Results}

In this section, we provide empirical results of our method. First, we provide a practical use case of the FIM, in selecting parameters for PHATE \cite{moon_visualizing_2019}. Specifically, PHATE and other diffusion-based methods use two prominent parameters, one that describes the bandwidth of the Gaussian (or related) kernel and another that describes the time of diffusion. We show how the FIM can be used to explore the parameter space by rendering this as a statistical manifold. Second, we apply neural FIM to single cell and toy datasets to showcase information about the data revealed by neural FIM including local volume, trace and eigenspectrum of the metric tensor at individual datapoints. Finally we demonstrate applications of our neural ODE by computing geodesics, which involves optimizing over path lengths of curves computed by the FIM. 

\subsection{Parameter selection for diffusion potentials with FIM}
\label{sec:PHATEparam}

As described in Sec.~\ref{sec: data_diffusion}, data diffusion is a powerful framework for exploring the manifold-intrinsic structure of a dataset based on exploring the data through an a Gaussian affinity matrix $K$ which is then normalized Markovian random walk process $P$ and powered to a diffusion time scale $t$ to mimic different steps of random walks. In PHATE, these parameters have significant effects on the visualization (See Figure \ref{fig:Volume_PHATE_parameters}).  Here we create a 2-dimensional statistical manifold consisting of parameters $\sigma$ corresponding to the bandwidth of the Gaussian Kernel, and the diffusion time scale, $t$. Though we only discuss the diffusion potential matrix rendered by PHATE, the reader should note that using the lens of the FIM for various transformations of point cloud data to notions of similarity or distance could be of interest to practitioners. 

We consider the FIM described in \Eqref{eq:FIM} with respect to the bandwidth and diffusion time scale parameters, $I_\Psi(\sigma,t)$, where $\Psi: \mathbb{R}^{N} \to \mathbb{R}^{N \times N}$ is used to construct the diffusion potential matrix. To achieve this, we first generate point cloud data using the tree dataset available in the PHATE package. We then subsampled points randomly and generate the diffusion potential matrix $\mathbf{P}^{t}_{ij} $ using techniques from Sec.~\ref{sec: data_diffusion}. To understand how $\theta = (t,\sigma)$ affect the construction of the diffusion potential matrix, we compute the volume of the $[2 \times 2]$ FIM using $V(M) = \int_M \sqrt{|det(I_\Psi(\sigma,t))|} ~d\theta$ for each combination of $\theta_k = (t_k,\sigma_m)$ for a finite range of values $t=[1,15]$, $m=[50,150]$. In Figure \ref{fig:Volume_PHATE_parameters}A, we show four different PHATE embeddings of the data corresponding to different parameter selections. We generally see that more details of branches are available in the embedding at lower values of $t$, and that higher values of bandwidth in this range retain more of the geometric structure. Thus, differences in these parameters have marked effects on the embedding. 
\vspace{-2mm}
\begin{figure}[H]
    \centering
    \includegraphics[width=1\columnwidth]{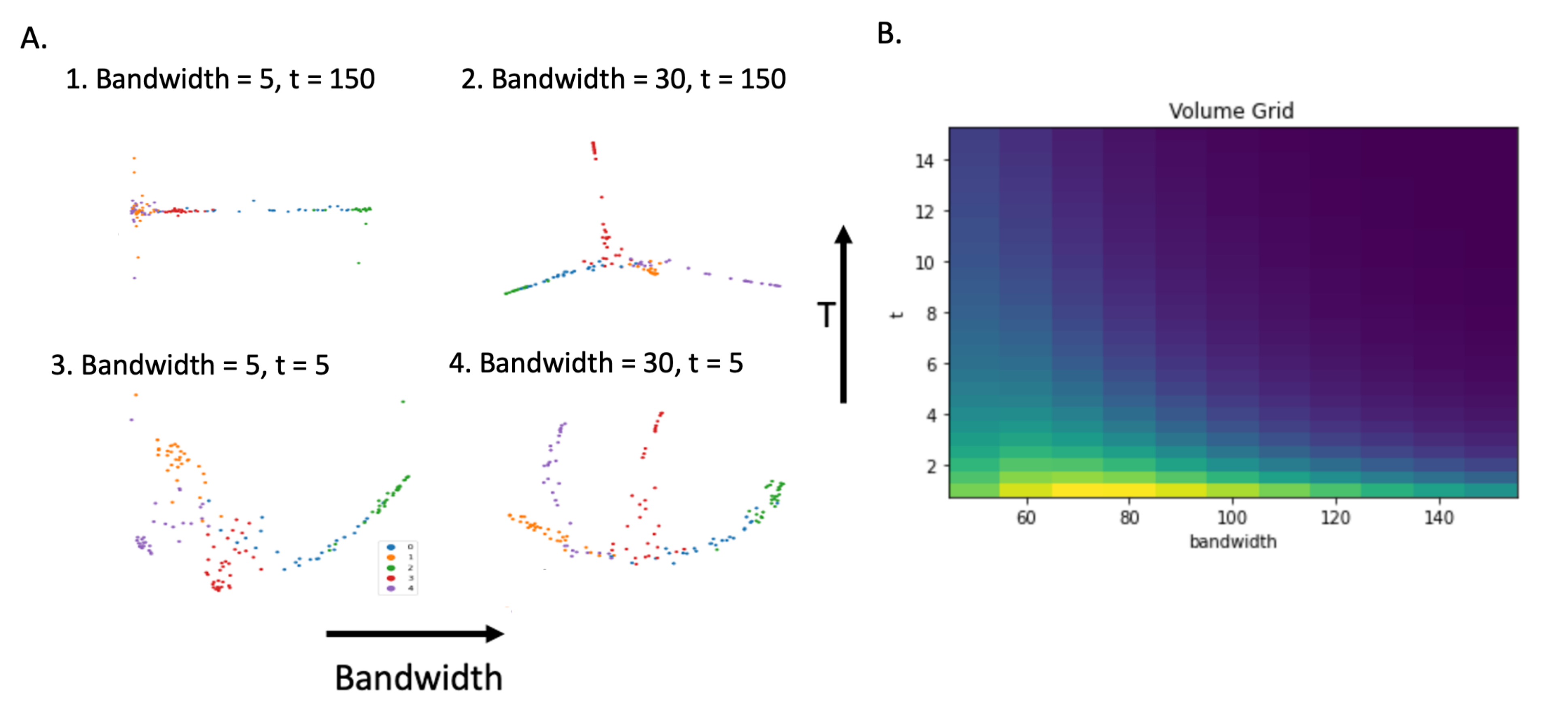}
    \caption{A. PHATE embeddings of the same artificial tree data with different bandwidth $\sigma$ and $t$ parameters. B. Volume of FIM with respect to $t$ and  $\sigma$.}
    \label{fig:Volume_PHATE_parameters}
    \vspace{-2mm}
\end{figure}

In Figure \ref{fig:Volume_PHATE_parameters}B we show a heat map of the FIM volume for different parameters of $\theta_k = (t_k,\sigma_k)$. From inspection, one can see certain combinations of $t$ and $\sigma$ yield the highest volume. In particular, brighter regions of the volume grid reveal the combinations of $t$ and $\sigma$ that undergo the most change from the point cloud space to the diffusion potential space while the converse is true for darker regions. This aligns with our intuition about how $t$ and $\sigma$ affect the diffusion potential construction of the graph: since the diffusion time scale and the bandwidth both incorporate trade-offs between local and global structure depending on their magnitude, we expect there to be a finite range in the parameter space where this trade-off is optimal.  In this case, we discover that this range is between 60 and 90 for $\sigma$ and between $0$ and $4$ for $t$. We also notice that the influence of the bandwidth depends on $t$; for smaller, $t$ a change of the bandwidth results in a larger change on the diffusion probabilities. 
 
\subsection{Neural FIM Embeddings}

Here, we deploy neural FIM on three datasets: 1) a toy tree dataset generated similarly to the one above, 2) a single cell mass cytometry dataset of induced pluripotent stem cell (IPSC) reprogramming \cite{zunder2015continuous} containing 220450 cells and 33 features, 3) a single cell RNA-sequencing dataset measuring peripheral blood mononucleocyte cells (i.e., immune cells) from a healthy donor (publicly available on the 10x website)~\cite{PBMCs} containing 2638 cells and 1838 features. For each dataset, we compute the FIM $g$ for each dataset, and we explore the point-wise trace $tr(g)$ and volume $\sqrt{|det(g)|}$ to understand manifold-intrinsic structure and geometric properties of our datasets. The embeddings for each dataset are generated by applying PHATE with a JSD between rows of the diffusion potential matrix in the ambient space (see \ref{fig:EmbeddingComp}. for comparison between PHATE and PHATE-JSD). 

 To obtain the continuous FIM for each dataset, we first compute the diffusion operator matrix $\mathbf{P}(x,\cdot) \in R^{n \times n}$ for each batch of point cloud data. We then embed each point in the batch $x_i \in R^{d}$ using the encoding network, neural-FIM $\phi: R^{d} \to R^{m} $ where $d$ is the original dimension of the point cloud data and $m$ is the last dimension of the encoder. Next, train the neural FIM using the loss defined in Equation~\ref{eq:JSLoss}.
 
 We can then compute the Jacobian  $\mathbf{J}(x_i) \in R^{m \times d}$ for each point of the network output  $x_i \in R^m$  with respect to the input coordinates. An FIM $I_\phi(x_i) \in R^{n \times n}$ can then be computed (using Equation \ref{eq:FIM}) for any point input to neural-FIM, thus yielding a continuous FIM for the manifold. Crucially, this allows one to compute information-theoretic and geometric quantities such as divergences (infinitesimally), volume, length, and relatedly---geodesics.

\subsubsection{Toy Data}

The artificial tree dataset we use for this was randomly generated using a built-in function in the PHATE package \cite{moon_visualizing_2019} which allows one to generate random trees by specifying the number of branches and the number of dimensions.

 To validate the FIM computation, we color the embedding of the tree with the volume (Figure~\ref{fig:FIM_tree}A) and trace (Figure~\ref{fig:FIM_tree}B) which are now accessible with the continuous FIM. Intuitively, the magnitude of the volume and trace are high for regions of the data where there are several directions of progression available for datapoints corresponding to each of the branches. In such areas, the metric tensor has several high eigenvalues.  Conversely, the volume and trace will be low in regions of the manifold where there is a single direction of progression such as along individual branches. This relationship can be seen in Figures \ref{fig:FIM_tree}A and \ref{fig:FIM_tree}B---the region of the manifold where the branching occurs (in the center) contains the highest magnitude of volume and trace while the converse holds for sparse areas of the manifold. This variational coloring of trace and volume we observe empirically is a good sanity check that the neural-FIM network is computing what we expect and motivates us to move to real-world examples.

\begin{figure}[htb]
    \centering
    \includegraphics[width=1\columnwidth]{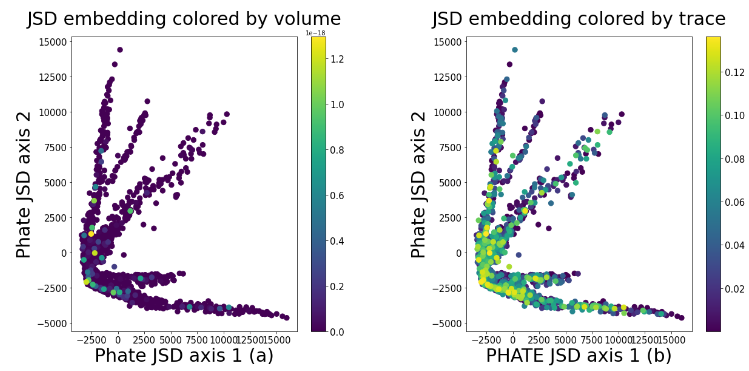}
    \caption{PHATE embedding of tree data colored by the volume (a) and trace (B) of the FIM. }
    \label{fig:FIM_tree}
    \vspace{-4mm}
\end{figure}

\subsubsection{Single Cell Data}

In Figures \ref{fig:FIM_pbmc}A and \ref{fig:FIM_pbmc}B the trace and the volume are colored on a 2D visualization of an embedding for peripheral blood mononucleic cell (PBMC) dataset. This dataset consists of three major classes of immune cells: T cells, B cells and Monocytes. In each cluster the center of the cluster has highest volume with boundaries having lower volume. Interestingly boundaries that are at the edges of the data (pointing away from other clusters) have even lower volume. Again, the volume seems to indicate potential choices for traveling along the manifold. Hence, the FIM can be used as a tool for illuminating regions of the manifold that retain the most informative components of the manifold.

\begin{figure}[htb]
    \centering
    \includegraphics[width=1\columnwidth]{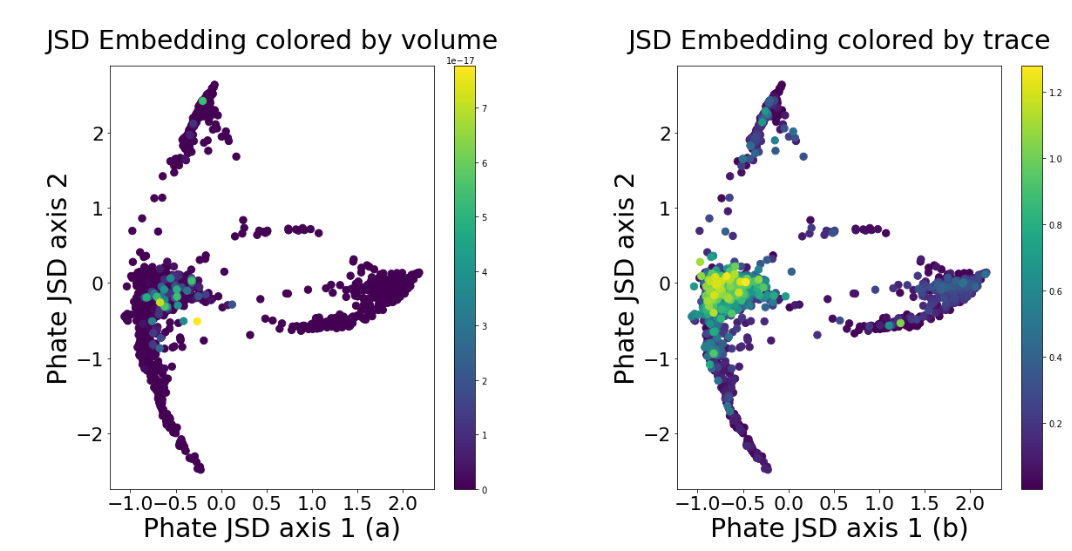}
    \caption{PHATE embedding of pbmc data colored by the volume (a) and trace (b) of the FIM.}
    \label{fig:FIM_pbmc}
    \vspace{-3mm}
\end{figure}

\vspace{-1mm}
\begin{figure}[htb]
    \centering
    \includegraphics[width=1\columnwidth]{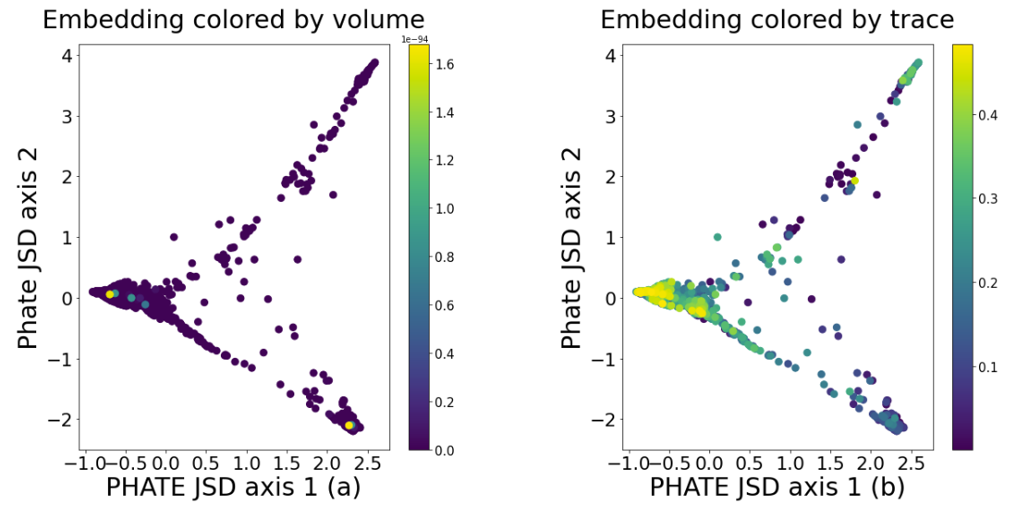}
    \caption{PHATE embedding of IPSC data colored by the volume (a) and trace (b) of the FIM.}
    \label{fig:FIM_ipsc}
    \vspace{-1mm}
\end{figure}

\vspace{-1mm}
\begin{figure}[htb]
    \centering
    \includegraphics[width=1\columnwidth]{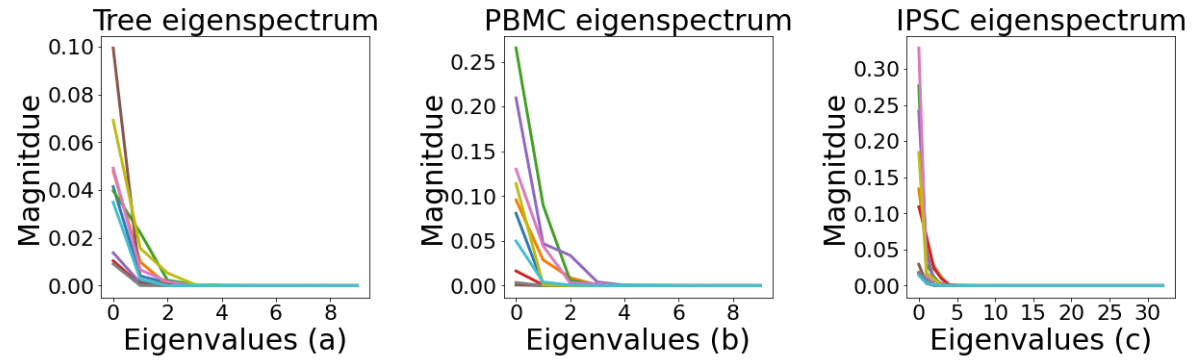}
    \caption{NeuralFIM eigenspectrums for different datasets: (a) Tree (b) PBMC (c) IPSC.}
    \label{fig:FIM_eigen}
    \vspace{-1mm}
\end{figure}

The same analysis was carried out for Induced Pluripotent Stem Cell (IPSC) data \cite{zunder2015continuous} measured using a different single cell technology: mass cytometry, which measures protein abundances. In this dataset, fibroblasts are being reprogrammed into pluripotent stem cells---a process that reverses natural differentiation (potentially for therapeutic purposes). The neural FIM embedding correctly shows a `Y' shape corresponding to the two branches described in \cite{zunder2015continuous}. One branch is successfully reprogramming and the other corresponds to failed reprogramming.  We again embedded points using neural-FIM on the IPSC data and color by volume and trace in Figures \ref{fig:FIM_ipsc}A and \ref{fig:FIM_ipsc}B, respectively. Here, the trace colored along the manifold reaches its highest magnitude where there are more axes of potential change and its lowest magnitude along the edges and tails where the datapoints do not have many directions to go. We also show the eigenspectrum of the FIM for all embeddings in Figures \ref{fig:FIM_eigen}A, \ref{fig:FIM_eigen}B, and \ref{fig:FIM_eigen}C which can have utility for discerning the number and index of relevant axes of information during the neural-FIM mapping. 
Using the same line of reasoning, we can look at the decay of the FIM eigenspectra for points located on different regions (e.g. sparse vs.\ dense) as well as the eigenvectors to extract insightful information about point cloud data.

\subsection{Learning Geodesics using Neural ODEs}

\begin{figure}[H]
    \centering
    \includegraphics[width=1\columnwidth]{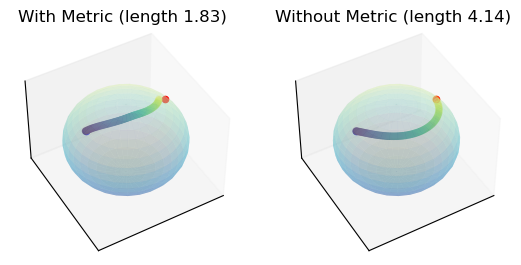}
    \caption{Learned paths from a Neural ODE, the path with the spherical metric learn a curve closer to the geodesic.}
    \label{fig:geo_sphere}
\end{figure}

We test the training objective defined in \Eqref{eq:ode_loss}, to learn the geodesic on a sphere of radius one. In spherical coordinates $(\theta,\psi)\in [0,\pi]\times[0,2\pi]$ we know the metric $ds^2 = d\theta^2 + \sin^2\theta d\psi^2$, and can thus compute the length of any path on the sphere. We train a neural ODE with three layers of width $64$ and SeLU activation function between each layer. We approximate the integration with the Runge-Kutta solver of order four. In Figure~\ref{fig:geo_sphere} we learn the path between two points above the equator $(\pi/4,0)$ to $(\pi/4,\pi)$; the geodesic passes closer to the north pole. We see that the path of the neural ODE trained to minimize the length indeed finds the right geodesic, while the one trained only to minimize the MSE with the final time point learns a longer path ($\theta$ appears to be constant along the path).

\begin{figure}[thb]
    \centering    \includegraphics[width=1\columnwidth,height=0.2\textheight]{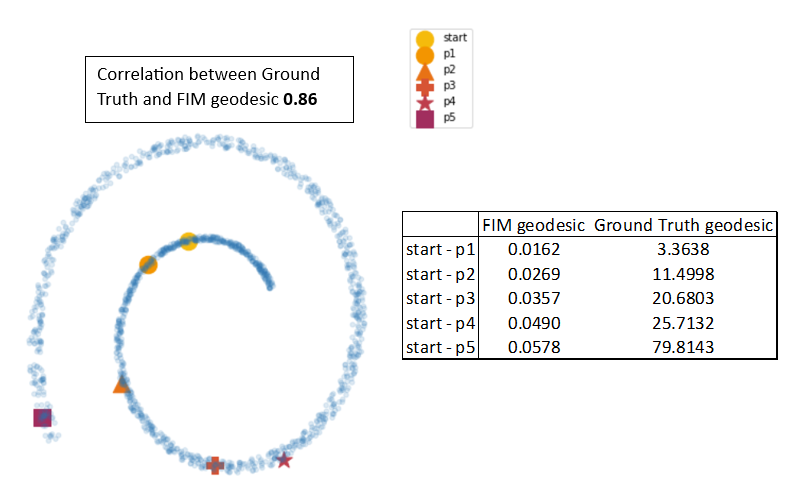}
    \caption{Ground Truth and NeuralFIM geodesics between a fixed point and 5 randomly selected points on swiss roll dataset}
    \label{fig:geo_swiss}
\end{figure}

In Figure~\ref{fig:geo_swiss} we validate our method of computing geodesics with the learned FIM's against ground truth geodesics on a swiss roll. For a quick sanity check, one can see that increasing the path length from the start to end point on the swiss roll corresponds to an increase in geodesic magnitude for the FIM and ground truth geodesic. Additionally, we see the FIM geodesic is strongly correlated with the ground truth geodesic.

\begin{figure}[thb]
    \centering
    \includegraphics[width=.7\columnwidth]{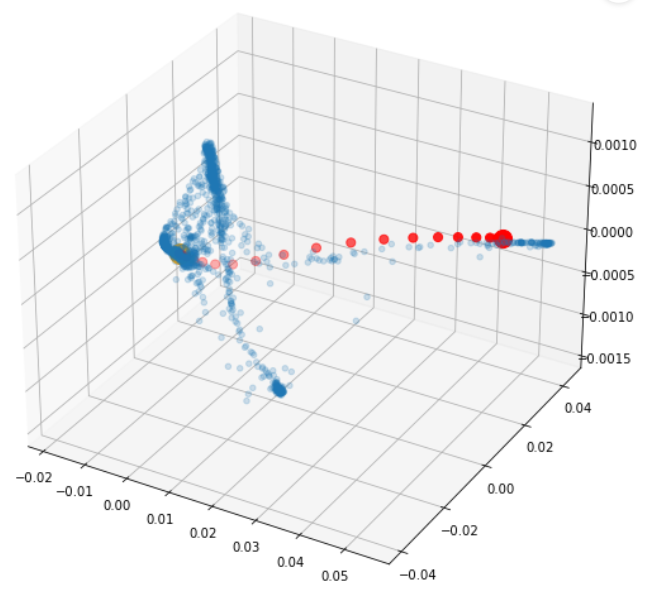}
    \caption{Learned neural ODE geodesic from the IPSC data connecting an initial cell to its final reprogrammed destination. Notably, the trajectory goes along the manifold. }
    \label{fig:geo_ipsc}
\end{figure}

Figure \ref{fig:geo_ipsc} shows a geodesic path that goes from a fibroblast cell (at the intersection of the branches) towards a reprogramming endpoint. We see that with only the endpoints specificed the neural ODE is able to find the path of reprogramming, i.e., as the geodesic path along the manifold. Generally, this could be highly useful for finding differentiation and progression paths from single cell data.  

\section{Conclusion}

Here we presented neural FIM, a novel method for learning a Riemannian metric from high dimensional point cloud data. We utilize FIM as a metric for data points represented by data diffusion probability distributions. Such distributions are computed via a Markovian diffusion operator which is used in diffusion maps, PHATE, diffusion pseudotime and other popular data science techniques. Neural FIM then allows us to compute underlying manifold information such as volume, and geodesic distances in this space in a way that is extensible to new datapoints. To compute geodesics in data space we introduce an auxiliary neural ODE network that minimizes length computed using the FIM on learned curve between two datapoints. We showcase neural FIM on PHATE parameter selection, and in finding the underlying manifold of toy data as well as single cell data.

\section{Acknowledgements}

This research was enabled in part by compute resources provided by Mila (mila.quebec) and Yale. It was partially funded and supported by ESP \textit{Mérite} [G.H.], CIFAR AI Chair [G.W.], NSERC Discovery grant 03267 [G.W.], NIH grants (1F30AI157270-01, R01HD100035, R01GM130847, R01GM135929) [G.W.,S.K.], NSF Career grant 2047856 [S.K.], the Chan-Zuckerberg Initiative grants CZF2019-182702 and CZF2019-002440 [S.K.], the Sloan Fellowship FG-2021-15883 [S.K.], and the Novo Nordisk grant GR112933 [S.K.]. The content provided here is solely the responsibility of the authors and does not necessarily represent the official views of the funding agencies. The funders had no role in study design, data collection \& analysis, decision to publish, or preparation of the manuscript.

\bibliography{Fisher_ICML_2023}
\bibliographystyle{icml2023}

\newpage
\appendix
\onecolumn
\section*{A. Sensitivity Analysis of neuralFIM hyperparameters}

Here, we complete a sensitivity analysis of a selection of neuralFIM hyperparameters on the tree dataset. We perturb the k-nearest neighbors (kNN), the noise level, and the Encoder Dimensions to understand whether our is robust with respect to the aforementioned hyperparameters. See Table \ref{tab: Sensitivity Analysis} and Figure \ref{fig:sensitvity_analysis} for empirical results.

\begin{figure}[H]
    \centering
    \includegraphics[scale=0.7]{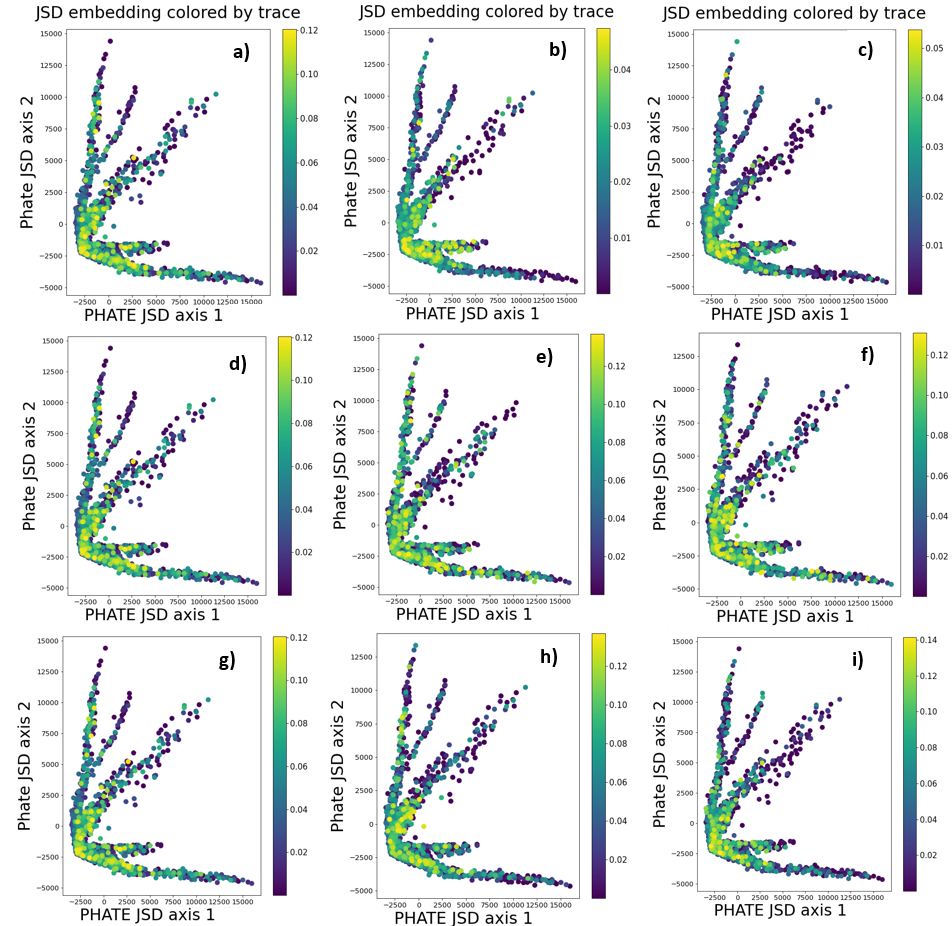}
    \caption{Visualization of PHATE-JSD embedding of tree dataset colored by FIM trace computed with NeuralFIM. We vary the kNN for the PHATE kernel, the noise level added to the input dataset and the dimensions of the autoencoder. One can readily see that neuralFIM is robust with respect to the perturbed hyperparameters. In row 1 we perturb the kNN across figures a-c [knn=5,10,15], respectively, in row 2 (d-f) we perturb the noise level added to the data input to the autoencoder [noise level = 0.0005,0.0010,0.0015], and in row 3 (g-i) we perturb the encoder dimension(ED) [ ED=100,100,50; ED=100,80,30; ED=100,70,20 ].}
    \label{fig:sensitvity_analysis}
    \vspace{-5mm}
\end{figure}

\begin{table}[H]

    \caption{Sensitivity Analysis: Correlation of FIM trace between encoder noise, embedding dimension, and KNN for tree dataset}
    \label{tab: Sensitivity Analysis}
    \centering
    \begin{tabular}{c|ccc | c|ccc | c|ccc}
    \toprule
        \multicolumn{4}{c}{\textbf{KNN}} & \multicolumn{4}{c}{\textbf{Noise Level}} & \multicolumn{4}{c}{\textbf{Embedding Dimension}} \\
        \midrule
          & \emph{5} & \emph{10} & \emph{15} &  & \emph{0.0005} & \emph{0.0010} & \emph{0.0015} &  & \emph{20} & \emph{30} & \emph{50} \\
          \hline
        \emph{5} & 1 & 0.9972 & 0.9989 & \emph{0.0005}& 1 & 0.9981 & 0.9990 & \emph{20} & 1 & 0.9971 & 0.9944\\
        \emph{10} & 0.9972& 1 & 0.9985 &\emph{0.0010} & 0.9981 & 1 & 0.9994 & \emph{30} & 0.9971 & 1 & 0.9985\\
        \emph{15} & 0.9989 & 0.9985 & 1 & \emph{0.0015} & 0.9990 & 0.9994 & 1 & \emph{50} & 0.9944 & 0.9985 & 1\\
        \bottomrule
    \end{tabular}

\end{table}

\subsection*{A.1 Algorithm}

Below we describe the probability distribution constructed for FIM computation. 

\begin{algorithm}[ht]
  \caption{Phate Fisher Information Distribution}
  \label{alg:phate_fim}
\begin{algorithmic}
\State {\bfseries Input:} $N \times d$ dataset $X$, matrix diffusion time $t$
\State {\bfseries Returns:} $N \times N$ diffusion potential matrix, $ \mU $
\State $\mD_{ij} \gets \|\mX_i - \mX_j\|_2$
\State $\mA \gets \mathrm{kernel}(\mD)$
\State $\mQ \gets \mathrm{Diag}(\mA \mathbf{1})$
\State $\mK \gets \mQ^{-1} \mA \mQ^{-1}$
\State $\mP \gets \mathrm{RowNormalize}(\mK)$
\LComment{Here $\log$ is applied elementwise, power is matrix power.}
\State \Return $ \mU \gets \log (\mP^{t})$
\end{algorithmic}
\end{algorithm}

\subsection*{A.2 Connection between FIM and KL divergence}

In \cite{cover2006elements} the Fisher Information Metric is derived as the second derivative of KL-divergence. To derive this the authors consider two probability distributions $P(x)$ and $P(y)$ that are infinitesimally close to one another. 
$P(y)=P(x)+\sum_j \Delta x_j \frac{\partial{P}}{\partial{x_j}}$ where $\Delta x_j$ is an infinitesimally small change of $x$ in the $j$ direction.

Since KL-divergences are $0$ when two distributions are equal to one another, they use a second order Taylor expansion of the KL-divergence as given by: 

$g_x(y)=KL(P(x)||P(y)=\frac{1}{2} \sum_{i,j} \Delta x_i \Delta x_j g_{ij}(x)$

\subsection*{A.3 JS Distance PHATE}

Typically the PHATE dimensionality reduction method works along the following steps: 
\begin{itemize}
    \item Compute diffusion operator $P$ from data using an alpha-decay kernel given in \cite{moon_visualizing_2019}. 
    \item Compute potential distances which are $M$-divergences between rows of the diffusion operator $P(x,\cdot)$ as given by $pdist(i,j)=\sqrt{\sum_k (\log(P^t(i,k)-P^t(j,k))}$ between points $i$ and $j$.
    \item Use the potential distance matrix as input to metric MDS to reduce to two dimensions. 
\end{itemize}

To train the neural FIM, we instead replace the $M$-divergence in PHATE with JS distance given by:

\begin{align*}
\scriptsize
&JS(\mP_n^t(i,\cdot),\mP_n^t(j,\cdot)):= \\ & \tfrac{1}{2} KL((\mP_n^t(i,\cdot)||M)+KL((\mP_n^t(j,\cdot)||M),
\end{align*}
where $M := (1/2) (\mP_n^t(i,\cdot)+\mP_n^t(j,\cdot))$.

Then we use the same metric MDS steps to reduce to an arbitrary $k$, though not typically $2$ dimensions. We note that this preserves manifold structure as well as or better than the originally proposed M-divergences. In Figure \ref{fig:EmbeddingComp} we show embeddings of artificially generated tree-structured data with original PHATE on the left and neural FIM trained with JSD-PHATE on the right. The embeddings look similar with the JSD embedding looking even more denoised than the original PHATE embedding.

\begin{figure}[H]
    \centering
    \includegraphics[scale=0.7]{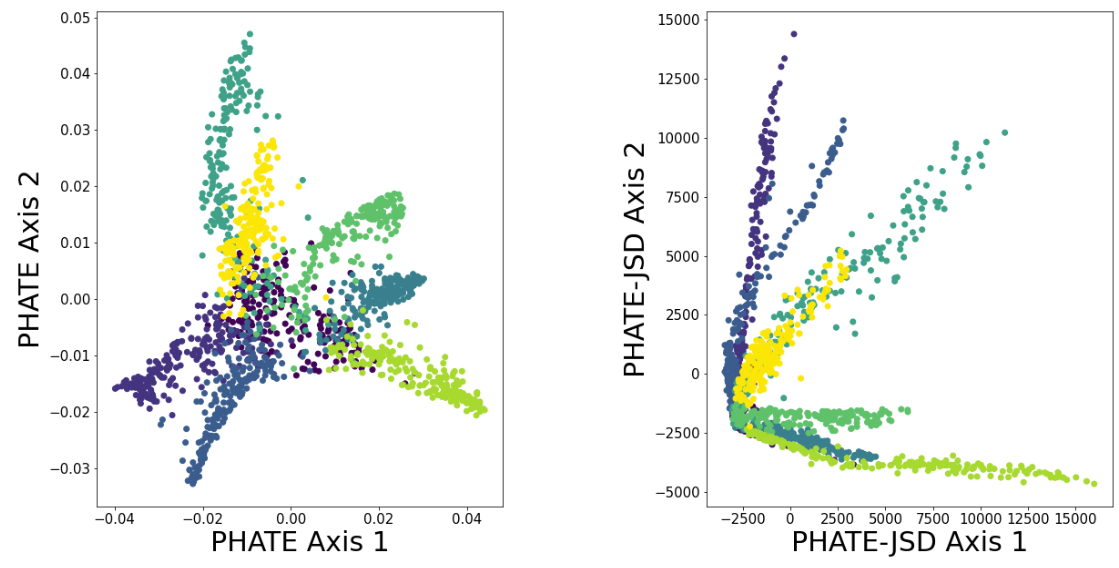}
    \caption{Embeddings of artificially generated tree data with PHATE and neural FIM trained with JSD-PHATE}
    \label{fig:EmbeddingComp}
    \vspace{-5mm}
\end{figure}

\subsection*{A.4 Experimental details}

 Each dataset was run with the same neural network parameters: Encoding Layers = [100,100,50] (for k=50); [100,80,30] (for k=30); [100,70,20] (for k=20) where k = latent dimensions, 150 epochs, ReLu activation between encoding layers, and using the AdamW optimizer with learning rate = 1e-4. For the neuralODE, we use 3 hidden layers [64,64,64] and use Runge-Kutta for the ODE solver. For the experimental results, we use 20 time steps between start and end point and train for 250 epochs again using the AdamW optimizer with learning rate = 1e-4.

\end{document}